\pdfoutput=1
\documentclass[sigconf]{acmart}

\copyrightyear{2026}
\acmYear{2026}
\setcopyright{cc}
\setcctype{by}
\acmConference[GECCO Companion '26]{Genetic and Evolutionary Computation Conference}{July 13--17, 2026}{San Jose, Costa Rica}
\acmBooktitle{Genetic and Evolutionary Computation Conference (GECCO Companion '26), July 13--17, 2026, San Jose, Costa Rica}
\acmDOI{10.1145/3795101.3805361}
\acmISBN{979-8-4007-2488-6/2026/07}

\usepackage{booktabs}
\usepackage{multirow}
\usepackage{xcolor}
\definecolor{oiblue}{HTML}{0072B2}
\definecolor{oiorange}{HTML}{E69F00}
\definecolor{oiskyblue}{HTML}{56B4E9}
\definecolor{oigray}{HTML}{999999}
\definecolor{oivermillion}{HTML}{D55E00}
\definecolor{oigreen}{HTML}{009E73}
\definecolor{oipurple}{HTML}{CC79A7}
\usepackage{graphicx}
\usepackage{amsmath}
\usepackage{amsthm}
\usepackage{tikz}
\usepackage{pgfplots}
\pgfplotsset{compat=1.18}
\usepgfplotslibrary{fillbetween}
\usetikzlibrary{shapes,arrows,positioning,calc,patterns,decorations.pathreplacing}

\newcommand{\hh}{h$\rightarrow$h}
\newcommand{\BigO}{\mathcal{O}}
\newtheorem{theorem}{Theorem}

\newcommand{\emrhn}{EMR-HyperNEAT}

\begin{document}

\title{Tensor-Accelerated Eager Multi-Resolution Grids for Evolving Large-Scale Substrates}

\author{Romain Claret}
\email{romain.claret@unine.ch}
\orcid{0000-0002-5612-8471}
\affiliation{%
  \institution{University of Neuch\^atel}
  \city{Neuch\^atel}
  \country{Switzerland}
}

\author{Michael O'Neill}
\email{m.oneill@ucd.ie}
\orcid{0000-0001-8734-417X}
\affiliation{%
  \institution{University College Dublin}
  \city{Dublin}
  \country{Ireland}
}

\author{Paul Cotofrei}
\email{paul.cotofrei@unine.ch}
\orcid{0000-0002-4103-5467}
\affiliation{%
  \institution{University of Neuch\^atel}
  \city{Neuch\^atel}
  \country{Switzerland}
}

\author{Kilian Stoffel}
\email{kilian.stoffel@unine.ch}
\orcid{0000-0002-9486-7769}
\affiliation{%
  \institution{University of Neuch\^atel}
  \city{Neuch\^atel}
  \country{Switzerland}
}

\renewcommand{\shortauthors}{Claret et al.}

\begin{abstract}
Evolvable-Substrate HyperNEAT (ES-HyperNEAT) discovers neural network topology by querying a Compositional Pattern-Producing Network (CPPN) at spatial positions and placing nodes where output variance is high. Its sequential quadtree cannot be tensorized directly: each network produces a different tree with dynamic shapes. We present Eager Multi-Resolution HyperNEAT (\emrhn{}), which reformulates adaptive substrate discovery as a batch tensor operation: shared position grids are precomputed for all depths, every position is evaluated in one vectorized CPPN call, then the same variance criterion filters the output. This tensor formulation supports population-level batching, one-time JIT compilation, chunked memory streaming, and six substrate configurations (feedforward through full recurrent) that sequential traversal could not support at scale. On the CRSP/Compustat financial dataset (94K samples), \emrhn{} achieves 5.5$\times$ speedup at depth~6; on XOR (pop 1000), 12--34$\times$ per-generation GPU speedup at depths 5--7 (${\sim}100\times$ over 30 generations). The reformulation also discovers a superset of ES-HyperNEAT's positions, yielding higher empirical solve rates. We validate substrate evolution up to depth~13 (358M positions, streamed from disk).
\end{abstract}

\ccsdesc[500]{Computing methodologies~Neural networks}
\ccsdesc[500]{Computing methodologies~Genetic algorithms}
\ccsdesc[500]{Computing methodologies~Parallel algorithms}

\keywords{neuroevolution, ES-HyperNEAT, adaptive substrates, tensor reformulation, indirect encoding, recurrent networks}

\maketitle

\section{Introduction}
\label{sec:introduction}

In neuroevolution, indirect encoding generates neural network connectivity from a compact genome rather than specifying each connection. ES-HyperNEAT~\cite{risi2012enhanced} automatically discovers where to place hidden nodes by examining CPPN~\cite{stanley2007compositional} output patterns: it recursively subdivides space using a quadtree, expanding regions where CPPN outputs show high variance. This adaptive approach discovers network topology without manual substrate specification, extending the fixed-grid HyperNEAT~\cite{stanley2009hypercube} framework built on NEAT~\cite{stanley2002evolving}.

However, the quadtree resists tensorization. Each depth level depends on the parent's variance, forcing sequential evaluation. Different CPPNs produce different subdivision patterns, preventing batching. And variable leaf counts are incompatible with JAX's~\cite{bradbury2018jax} static shape requirement for JIT compilation. Our prior work confirmed these limits at depths exceeding 5~\cite{claret2024investigating-full}, and a JAX reimplementation of the quadtree yielded only marginal speedup despite batched optimizations, motivating the eager reformulation presented here.

We present \emrhn{}, which evaluates all positions at all resolutions up front, then filters using the same variance criterion: ES-HyperNEAT's \texttt{subdivide\_if(var~>~$\theta$)} becomes \texttt{eval\_all(); filter(var~>~$\theta$)} (Figure~\ref{fig:lazy_vs_eager}). This performs more CPPN queries than necessary, but all queries become independent and parallelizable across both cores and population members, reducing complexity from $\BigO(4^D)$ to $\BigO(4^D/P)$ across~$P$ parallel cores (Section~\ref{sec:algorithm}). Recurrent substrate configurations become feasible through a connection type taxonomy (Section~\ref{sec:recurrence}). Section~\ref{sec:experiments} validates 12--34$\times$ on-device GPU speedup on XOR at depths~5--7, and empirically higher solve rates across benchmarks.

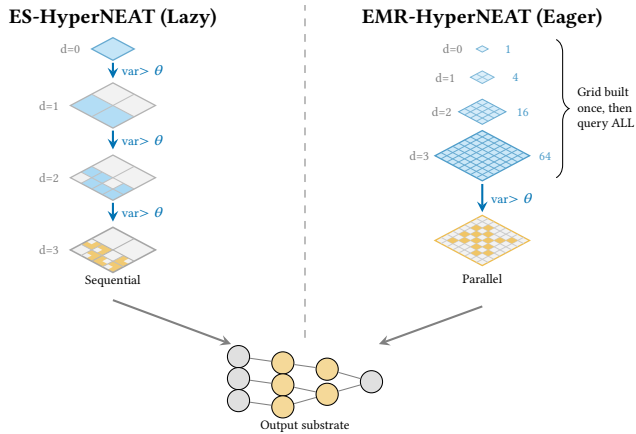
\begin{figure}[b!]
\centering
\resizebox{\columnwidth}{!}{%
\begin{tikzpicture}[yscale=0.85,
    node distance=0.4cm,
    netnode/.style={circle, draw, fill=oiskyblue!20, minimum size=3mm, inner sep=0pt},
    arrow/.style={->, >=stealth, thick},
    brace/.style={decorate, decoration={brace, amplitude=5pt, mirror}}
]

\pgfmathsetmacro{\ux}{0.08}
\pgfmathsetmacro{\uy}{0.05}

\begin{scope}[shift={(-2.6,0.6)}]
    \node[font=\small\bfseries] at (0,2.0) {ES-HyperNEAT (Lazy)};
    \pgfmathsetmacro{\sx}{1.8}
    \begin{scope}[shift={(0,1.5)}]
        \fill[oiskyblue!35] (0,{2*\uy*\sx}) -- ({2*\ux*\sx},0) -- (0,{-2*\uy*\sx}) -- ({-2*\ux*\sx},0) -- cycle;
        \draw[oiblue!60, line width=0.5pt] (0,{2*\uy*\sx}) -- ({2*\ux*\sx},0) -- (0,{-2*\uy*\sx}) -- ({-2*\ux*\sx},0) -- cycle;
        \node[font=\tiny, gray] at ({-4*\ux*\sx},0) {d=0};
    \end{scope}
    \draw[arrow, oiblue] (0,1.28) -- node[right, font=\tiny] {var$>\theta$} (0,1.03);
    \begin{scope}[shift={(0,0.6)}]
        \fill[oiskyblue!45] (0,0) -- ({-2*\ux*\sx},{2*\uy*\sx}) -- ({-4*\ux*\sx},0) -- ({-2*\ux*\sx},{-2*\uy*\sx}) -- cycle;
        \fill[gray!10] (0,0) -- ({2*\ux*\sx},{2*\uy*\sx}) -- ({4*\ux*\sx},0) -- ({2*\ux*\sx},{-2*\uy*\sx}) -- cycle;
        \fill[gray!10] (0,0) -- ({-2*\ux*\sx},{2*\uy*\sx}) -- (0,{4*\uy*\sx}) -- ({2*\ux*\sx},{2*\uy*\sx}) -- cycle;
        \fill[oiskyblue!45] (0,0) -- ({2*\ux*\sx},{-2*\uy*\sx}) -- (0,{-4*\uy*\sx}) -- ({-2*\ux*\sx},{-2*\uy*\sx}) -- cycle;
        \draw[gray!60, line width=0.5pt] (0,{4*\uy*\sx}) -- ({4*\ux*\sx},0) -- (0,{-4*\uy*\sx}) -- ({-4*\ux*\sx},0) -- cycle;
        \draw[gray!50, line width=0.4pt] ({-2*\ux*\sx},{2*\uy*\sx}) -- ({2*\ux*\sx},{-2*\uy*\sx});
        \draw[gray!50, line width=0.4pt] ({2*\ux*\sx},{2*\uy*\sx}) -- ({-2*\ux*\sx},{-2*\uy*\sx});
        \node[font=\tiny, gray] at ({-6*\ux*\sx},0) {d=1};
    \end{scope}
    \draw[arrow, oiblue] (0,0.2) -- node[right, font=\tiny] {var$>\theta$} (0,-0.1);
    \begin{scope}[shift={(0,-0.55)}]
        \fill[oiskyblue!50] ({-2*\ux*\sx},0) -- ({-3*\ux*\sx},{\uy*\sx}) -- ({-2*\ux*\sx},{2*\uy*\sx}) -- ({-\ux*\sx},{\uy*\sx}) -- cycle;
        \fill[gray!10] ({-2*\ux*\sx},0) -- ({-3*\ux*\sx},{-\uy*\sx}) -- ({-4*\ux*\sx},0) -- ({-3*\ux*\sx},{\uy*\sx}) -- cycle;
        \fill[oiskyblue!50] ({-2*\ux*\sx},0) -- ({-\ux*\sx},{-\uy*\sx}) -- ({-2*\ux*\sx},{-2*\uy*\sx}) -- ({-3*\ux*\sx},{-\uy*\sx}) -- cycle;
        \fill[oiskyblue!50] ({-2*\ux*\sx},0) -- ({-\ux*\sx},{\uy*\sx}) -- (0,0) -- ({-\ux*\sx},{-\uy*\sx}) -- cycle;
        \draw[gray!50, line width=0.25pt] ({-\ux*\sx},{\uy*\sx}) -- ({-3*\ux*\sx},{-\uy*\sx});
        \draw[gray!50, line width=0.25pt] ({-3*\ux*\sx},{\uy*\sx}) -- ({-\ux*\sx},{-\uy*\sx});
        \fill[gray!10] (0,0) -- ({2*\ux*\sx},{2*\uy*\sx}) -- ({4*\ux*\sx},0) -- ({2*\ux*\sx},{-2*\uy*\sx}) -- cycle;
        \fill[gray!10] (0,0) -- ({-2*\ux*\sx},{2*\uy*\sx}) -- (0,{4*\uy*\sx}) -- ({2*\ux*\sx},{2*\uy*\sx}) -- cycle;
        \fill[gray!10] (0,0) -- ({\ux*\sx},{-\uy*\sx}) -- (0,{-2*\uy*\sx}) -- ({-\ux*\sx},{-\uy*\sx}) -- cycle;
        \fill[oiskyblue!50] (0,{-2*\uy*\sx}) -- ({\ux*\sx},{-\uy*\sx}) -- ({2*\ux*\sx},{-2*\uy*\sx}) -- ({\ux*\sx},{-3*\uy*\sx}) -- cycle;
        \fill[oiskyblue!50] (0,{-2*\uy*\sx}) -- ({-\ux*\sx},{-3*\uy*\sx}) -- (0,{-4*\uy*\sx}) -- ({\ux*\sx},{-3*\uy*\sx}) -- cycle;
        \fill[oiskyblue!50] (0,{-2*\uy*\sx}) -- ({-\ux*\sx},{-\uy*\sx}) -- ({-2*\ux*\sx},{-2*\uy*\sx}) -- ({-\ux*\sx},{-3*\uy*\sx}) -- cycle;
        \draw[gray!50, line width=0.25pt] ({\ux*\sx},{-\uy*\sx}) -- ({-\ux*\sx},{-3*\uy*\sx});
        \draw[gray!50, line width=0.25pt] ({-\ux*\sx},{-\uy*\sx}) -- ({\ux*\sx},{-3*\uy*\sx});
        \draw[gray!60, line width=0.5pt] (0,{4*\uy*\sx}) -- ({4*\ux*\sx},0) -- (0,{-4*\uy*\sx}) -- ({-4*\ux*\sx},0) -- cycle;
        \draw[gray!50, line width=0.4pt] ({-2*\ux*\sx},{2*\uy*\sx}) -- ({2*\ux*\sx},{-2*\uy*\sx});
        \draw[gray!50, line width=0.4pt] ({2*\ux*\sx},{2*\uy*\sx}) -- ({-2*\ux*\sx},{-2*\uy*\sx});
        \node[font=\tiny, gray] at ({-6*\ux*\sx},0) {d=2};
    \end{scope}
    \draw[arrow, oiblue] (0,-0.95) -- node[right, font=\tiny] {var$>\theta$} (0,-1.25);
    \begin{scope}[shift={(0,-1.7)}]
        \fill[oiorange!55] (0,0) -- ({-2*\ux*\sx},{2*\uy*\sx}) -- ({-4*\ux*\sx},0) -- ({-2*\ux*\sx},{-2*\uy*\sx}) -- cycle;
        \fill[gray!10] ({-2*\ux*\sx},0) -- ({-3*\ux*\sx},{-\uy*\sx}) -- ({-4*\ux*\sx},0) -- ({-3*\ux*\sx},{\uy*\sx}) -- cycle;
        \fill[gray!10] ({-2*\ux*\sx},{\uy*\sx}) -- ({-2.5*\ux*\sx},{1.5*\uy*\sx}) -- ({-2*\ux*\sx},{2*\uy*\sx}) -- ({-1.5*\ux*\sx},{1.5*\uy*\sx}) -- cycle;
        \fill[gray!10] ({-2*\ux*\sx},{-\uy*\sx}) -- ({-2.5*\ux*\sx},{-0.5*\uy*\sx}) -- ({-3*\ux*\sx},{-\uy*\sx}) -- ({-2.5*\ux*\sx},{-1.5*\uy*\sx}) -- cycle;
        \fill[gray!10] ({-\ux*\sx},0) -- ({-1.5*\ux*\sx},{0.5*\uy*\sx}) -- ({-2*\ux*\sx},0) -- ({-1.5*\ux*\sx},{-0.5*\uy*\sx}) -- cycle;
        \fill[gray!10] (0,0) -- ({2*\ux*\sx},{2*\uy*\sx}) -- ({4*\ux*\sx},0) -- ({2*\ux*\sx},{-2*\uy*\sx}) -- cycle;
        \fill[gray!10] (0,0) -- ({-2*\ux*\sx},{2*\uy*\sx}) -- (0,{4*\uy*\sx}) -- ({2*\ux*\sx},{2*\uy*\sx}) -- cycle;
        \fill[oiorange!55] (0,0) -- ({2*\ux*\sx},{-2*\uy*\sx}) -- (0,{-4*\uy*\sx}) -- ({-2*\ux*\sx},{-2*\uy*\sx}) -- cycle;
        \fill[gray!10] (0,0) -- ({\ux*\sx},{-\uy*\sx}) -- (0,{-2*\uy*\sx}) -- ({-\ux*\sx},{-\uy*\sx}) -- cycle;
        \fill[gray!10] ({\ux*\sx},{-2*\uy*\sx}) -- ({1.5*\ux*\sx},{-1.5*\uy*\sx}) -- ({2*\ux*\sx},{-2*\uy*\sx}) -- ({1.5*\ux*\sx},{-2.5*\uy*\sx}) -- cycle;
        \fill[gray!10] (0,{-3*\uy*\sx}) -- ({-0.5*\ux*\sx},{-2.5*\uy*\sx}) -- ({-\ux*\sx},{-3*\uy*\sx}) -- ({-0.5*\ux*\sx},{-3.5*\uy*\sx}) -- cycle;
        \fill[gray!10] ({-\ux*\sx},{-2*\uy*\sx}) -- ({-0.5*\ux*\sx},{-1.5*\uy*\sx}) -- (0,{-2*\uy*\sx}) -- ({-0.5*\ux*\sx},{-2.5*\uy*\sx}) -- cycle;
        \draw[gray!70, line width=0.6pt] (0,{4*\uy*\sx}) -- ({4*\ux*\sx},0) -- (0,{-4*\uy*\sx}) -- ({-4*\ux*\sx},0) -- cycle;
        \draw[gray!60, line width=0.4pt] ({-2*\ux*\sx},{2*\uy*\sx}) -- ({2*\ux*\sx},{-2*\uy*\sx});
        \draw[gray!60, line width=0.4pt] ({2*\ux*\sx},{2*\uy*\sx}) -- ({-2*\ux*\sx},{-2*\uy*\sx});
        \draw[gray!50, line width=0.3pt] ({-\ux*\sx},{\uy*\sx}) -- ({-3*\ux*\sx},{-\uy*\sx});
        \draw[gray!50, line width=0.3pt] ({-3*\ux*\sx},{\uy*\sx}) -- ({-\ux*\sx},{-\uy*\sx});
        \draw[gray!50, line width=0.3pt] ({\ux*\sx},{-\uy*\sx}) -- ({-\ux*\sx},{-3*\uy*\sx});
        \draw[gray!50, line width=0.3pt] ({-\ux*\sx},{-\uy*\sx}) -- ({\ux*\sx},{-3*\uy*\sx});
        \node[font=\tiny, gray] at ({-6*\ux*\sx},0) {d=3};
    \end{scope}
    \node[font=\tiny, align=center] at (0,-2.2) {Sequential};
\end{scope}

\begin{scope}[shift={(2.4,0.6)}]
    \node[font=\small\bfseries] at (0,2.0) {\emrhn{} (Eager)};
    \begin{scope}[shift={(0,1.5)}]
        \fill[oiskyblue!15] (0,\uy) -- (\ux,0) -- (0,-\uy) -- (-\ux,0) -- cycle;
        \draw[oiblue!50, line width=0.3pt] (0,\uy) -- (\ux,0) -- (0,-\uy) -- (-\ux,0) -- cycle;
        \node[font=\tiny, oiblue!70] at (0.35,0) {1};
        \node[font=\tiny, gray] at (-0.4,0) {d=0};
    \end{scope}
    \begin{scope}[shift={(0,1.05)}]
        \fill[oiskyblue!20] (0,{2*\uy}) -- ({2*\ux},0) -- (0,{-2*\uy}) -- ({-2*\ux},0) -- cycle;
        \draw[oiblue!50, line width=0.3pt] (0,{2*\uy}) -- ({2*\ux},0) -- (0,{-2*\uy}) -- ({-2*\ux},0) -- cycle;
        \draw[oiblue!40, line width=0.2pt] ({-\ux},\uy) -- (\ux,-\uy);
        \draw[oiblue!40, line width=0.2pt] ({-\ux},-\uy) -- (\ux,\uy);
        \node[font=\tiny, oiblue!70] at (0.45,0) {4};
        \node[font=\tiny, gray] at (-0.5,0) {d=1};
    \end{scope}
    \begin{scope}[shift={(0,0.5)}]
        \fill[oiskyblue!25] (0,{4*\uy}) -- ({4*\ux},0) -- (0,{-4*\uy}) -- ({-4*\ux},0) -- cycle;
        \draw[oiblue!60, line width=0.4pt] (0,{4*\uy}) -- ({4*\ux},0) -- (0,{-4*\uy}) -- ({-4*\ux},0) -- cycle;
        \foreach \k in {1,2,3} {
            \draw[oiblue!50, line width=0.2pt] ({-4*\ux+\k*\ux},{\k*\uy}) -- ({\k*\ux},{-4*\uy+\k*\uy});
            \draw[oiblue!50, line width=0.2pt] ({-4*\ux+\k*\ux},{-\k*\uy}) -- ({\k*\ux},{4*\uy-\k*\uy});
        }
        \node[font=\tiny, oiblue!70] at (0.55,0) {16};
        \node[font=\tiny, gray] at (-0.55,0) {d=2};
    \end{scope}
    \begin{scope}[shift={(0,-0.2)}]
        \fill[oiskyblue!30] (0,{8*\uy}) -- ({8*\ux},0) -- (0,{-8*\uy}) -- ({-8*\ux},0) -- cycle;
        \draw[oiblue!70, line width=0.5pt] (0,{8*\uy}) -- ({8*\ux},0) -- (0,{-8*\uy}) -- ({-8*\ux},0) -- cycle;
        \foreach \k in {1,...,7} {
            \draw[oiblue!60, line width=0.15pt] ({-8*\ux+\k*\ux},{\k*\uy}) -- ({\k*\ux},{-8*\uy+\k*\uy});
            \draw[oiblue!60, line width=0.15pt] ({-8*\ux+\k*\ux},{-\k*\uy}) -- ({\k*\ux},{8*\uy-\k*\uy});
        }
        \node[font=\tiny, oiblue!70] at (0.85,0) {64};
        \node[font=\tiny, gray] at (-0.85,0) {d=3};
    \end{scope}
    \draw[brace] (1.0,-0.55) -- (1.0,1.65) node[midway, right=5pt, font=\tiny, align=left] {Grid built\\once, then\\query ALL};
    \draw[arrow, oiblue] (0,-0.65) -- node[right, font=\tiny] {var$>\theta$} (0,-1.1);
    \begin{scope}[shift={(0,-1.55)}]
        \fill[gray!10] (0,{8*\uy}) -- ({8*\ux},0) -- (0,{-8*\uy}) -- ({-8*\ux},0) -- cycle;
        \draw[oiorange!70, line width=0.5pt] (0,{8*\uy}) -- ({8*\ux},0) -- (0,{-8*\uy}) -- ({-8*\ux},0) -- cycle;
        \foreach \k in {1,...,7} {
            \draw[gray!40, line width=0.1pt] ({-8*\ux+\k*\ux},{\k*\uy}) -- ({\k*\ux},{-8*\uy+\k*\uy});
            \draw[gray!40, line width=0.1pt] ({-8*\ux+\k*\ux},{-\k*\uy}) -- ({\k*\ux},{8*\uy-\k*\uy});
        }
        \foreach \i/\j in {1/2, 2/1, 2/3, 3/2, 3/4, 4/3, 4/5, 5/4, 5/6, 6/5, 2/5, 5/2, 1/6, 6/1, 3/6} {
            \pgfmathsetmacro{\cx}{(\i-\j)*\ux}
            \pgfmathsetmacro{\cy}{(\i+\j-7)*\uy}
            \fill[oiorange!60] ({\cx},{\cy+\uy}) -- ({\cx+\ux},{\cy}) -- ({\cx},{\cy-\uy}) -- ({\cx-\ux},{\cy}) -- cycle;
        }
    \end{scope}
    \node[font=\tiny, align=center] at (0,-2.15) {Parallel};
\end{scope}

\draw[gray, dashed] (0,-2.5) -- (0,2.8);

\begin{scope}[shift={(0,-2.9)}]
    \draw[arrow, gray] (-2.6,1.0) -- (-1.0,0.3);
    \draw[arrow, gray] (2.4,0.9) -- (1.0,0.3);
    \node[netnode, fill=oigray!30] (i1) at (-0.9,0.1) {};
    \node[netnode, fill=oigray!30] (i2) at (-0.9,-0.25) {};
    \node[netnode, fill=oigray!30] (i3) at (-0.9,-0.6) {};
    \node[netnode, fill=oiorange!40] (h1a) at (-0.3,0.0) {};
    \node[netnode, fill=oiorange!40] (h1b) at (-0.3,-0.35) {};
    \node[netnode, fill=oiorange!40] (h1c) at (-0.3,-0.7) {};
    \node[netnode, fill=oiorange!40] (h2a) at (0.3,-0.1) {};
    \node[netnode, fill=oiorange!40] (h2b) at (0.3,-0.5) {};
    \node[netnode, fill=oigray!30] (o1) at (0.9,-0.3) {};
    \draw[gray, line width=0.4pt] (i1) -- (h1a); \draw[gray, line width=0.4pt] (i2) -- (h1b);
    \draw[gray, line width=0.4pt] (i3) -- (h1c); \draw[gray, line width=0.4pt] (h1a) -- (h2a);
    \draw[gray, line width=0.4pt] (h1b) -- (h2b); \draw[gray, line width=0.4pt] (h1c) -- (h2b);
    \draw[gray, line width=0.4pt] (h2a) -- (o1); \draw[gray, line width=0.4pt] (h2b) -- (o1);
    \node[font=\tiny, align=center] at (0,-1.0) {Output substrate};
\end{scope}
\end{tikzpicture}%
}
\Description{Comparison of lazy (ES-HyperNEAT) and eager (EMR-HyperNEAT) substrate discovery approaches producing sparse neural network substrates.}
\caption{Lazy vs eager substrate discovery. Left: ES-HyperNEAT subdivides sequentially. Right: \emrhn{} evaluates all depths in parallel, then filters by variance.}
\label{fig:lazy_vs_eager}
\end{figure}

\section{Background}
\label{sec:background}

NEAT~\cite{stanley2002evolving} evolves neural network topology and weights simultaneously. HyperNEAT~\cite{stanley2009hypercube} extends NEAT with indirect encoding: a CPPN maps spatial coordinates $(x_1, y_1, x_2, y_2)$ to connection weights, so compact genomes can specify large-scale networks. ES-HyperNEAT~\cite{risi2012enhanced} removes the need for a predefined substrate by adaptively placing nodes via quadtree subdivision. In Phase~1, for each input coordinate, the quadtree expands regions where CPPN variance exceeds threshold $\theta_{div}$, discovering $H$ hidden positions. In Phase~2, each of the $H$ positions undergoes its own quadtree traversal to discover hidden-to-hidden connections, requiring $\BigO(H^2)$ sequential queries. Since cost grows as $4^D$, each additional depth roughly quadruples generation time.

The quadtree traversal resists tensor batching for three reasons: (1)~non-uniform subdivision patterns prevent batching via \texttt{vmap}; (2)~children cannot be evaluated until parent variance is computed; and (3)~variable leaf counts are incompatible with JAX's static shape requirement for JIT compilation.

\section{\emrhn{} Algorithm}
\label{sec:algorithm}

\emrhn{} replaces ES-HyperNEAT's sequential quadtree with three parallelizable stages: (1)~precompute static multiresolution grids, (2)~batch-evaluate all CPPN queries via \texttt{vmap}, and (3)~apply hierarchical variance filtering.

\subsection{Hierarchical Grid and Batch Query}

At initialization, \emrhn{} precomputes a complete multiresolution grid from depth 0 to maximum depth $D$:
\begin{equation}
\mathcal{G}_D = \bigcup_{d=0}^{D} \text{Grid}_d, \; \text{Grid}_d = \left\{\left(\tfrac{2i+1}{2^{d+1}} {-} 1, \tfrac{2j+1}{2^{d+1}} {-} 1\right) : i, j \in [0, 2^{d+1})\right\}
\end{equation}

\noindent The total position count is $N_{total} = (4^{D+2} - 4)/3$, growing from 1,364 at depth~4 to 87,380 at depth~7 and 358M at depth~13. All arrays have static shapes, compatible with JIT compilation. With positions in static arrays, CPPN queries become independent and parallel via $\mathbf{W} = \text{vmap}(\text{CPPN})(s, T)$, replacing $N$ sequential queries with $\BigO(N/P)$ wall-clock time. Population parallelization evaluates all genomes simultaneously.

\subsection{Variance Filtering and Active Mask}

Variance is computed bottom-up from children to parents. The active mask propagates top-down:
\begin{equation}
\text{active}[d][i] = \begin{cases}
\text{True} & d = 0 \\
\text{active}[d{-}1][p_i] \land (\text{var}[d{-}1][p_i] > \theta_{div}) & \text{else}
\end{cases}
\end{equation}

\noindent Each parent controls all four children as a block, so filtering runs in parallel. Table~\ref{tab:complexity} summarizes the complexity comparison.

\begin{table}[ht]
\centering
\caption{Computational complexity. $D$ = max depth, $H$ = discovered hidden positions, $P$ = parallel execution units.}
\label{tab:complexity}
\small
\begin{tabular}{lll}
\toprule
Operation & ES-HyperNEAT & \emrhn{} \\
\midrule
Grid Construction & --- & $\BigO(4^D)$ once \\
CPPN Query (Phase 1) & $\BigO(4^D)$ seq. & $\BigO(4^D / P)$ \\
Variance Comp. & $\BigO(4^D)$ seq. & $\BigO(4^D / P)$ \\
Mask Propagation & $\BigO(4^D)$ seq. & $\BigO(4^D / P)$ \\
\hh{} Query (Phase 2) & $\BigO(H^2)$ seq. & $\BigO(H^2 / P)$ \\
\midrule
\textbf{Total} & $\BigO(4^D + H^2)$ & $\BigO((4^D + H^2) / P)$ \\
\bottomrule
\end{tabular}
\end{table}

\begin{theorem}[Superset Discovery]
\label{thm:correctness}
Under identical variance thresholds $\theta_{div}$ and band thresholds $\theta_{band}$, \emrhn{} discovers a superset of the positions discovered by ES-HyperNEAT. When both algorithms discover identical position sets, they produce identical connection sets.
\end{theorem}

\begin{proof}
ES-HyperNEAT's quadtree prunes subtrees when parent variance falls below $\theta_{div}$. Since low parent variance does not imply low child variance, high-variance children within low-variance parents are never evaluated. \emrhn{}'s eager evaluation computes all positions before applying variance masks, recovering these positions. Therefore every position discovered by ES-HyperNEAT is also discovered by \emrhn{} (superset property). This also handles \texttt{initial\_depth=0} robustly: ES-HyperNEAT may discover no nodes if root variance is insufficient, while \emrhn{} evaluates all positions regardless.
\end{proof}

\subsection{Threshold Adaptation}

The two approaches compute variance over different regions: ES-Hyper\-NEAT only looks inside regions whose parent already passed the threshold, while \emrhn{} looks at all positions including low-variance regions.
ES-HyperNEAT's original thresholds ($\theta_{div}{=}0.5$, $\theta_{band}{=}0.3$) produce no solutions under eager evaluation. \emrhn{} uses adapted thresholds ($\theta_{var}{=}0.03$, weight magnitude filtering at~0.1) that are functionally equivalent for the eager context. This adaptation is a necessary consequence of the reformulation, not an independent parameter choice.

\section{Connection Type Taxonomy}
\label{sec:recurrence}

\emrhn{}'s parallelization of Phase~2 makes recurrent substrate configurations computationally feasible. In HyperNEAT substrates, comparing source and target Y coordinates classifies connections into four primitives: \textbf{Forward} ($y_s < y_t$), \textbf{Backward} ($y_s > y_t$, feedback), \textbf{Lateral} ($y_s = y_t$, same layer), and \textbf{Self-loop} ($s = t$). Enabling or disabling these produces six substrate configurations (Table~\ref{tab:connection_configs}).

\begin{table}[ht]
\centering
\small
\caption{Six substrate configurations available in \emrhn{}. F=Forward, H=hidden-to-hidden, B=Backward, L=Lateral, S=Self.}
\label{tab:connection_configs}
\begin{tabular}{@{}lcccccl@{}}
\toprule
\textbf{Config} & \textbf{F} & \textbf{H} & \textbf{B} & \textbf{L} & \textbf{S} & \textbf{Notes} \\
\midrule
Feedforward & \checkmark & & & & & Fastest; skips Phase~2 \\
\hh{} & \checkmark & \checkmark & & & & ES-HN default \\
Backward & \checkmark & \checkmark & \checkmark & & & Feedback loops \\
Lateral & \checkmark & \checkmark & & \checkmark & & Same-layer links \\
Self & \checkmark & \checkmark & & & \checkmark & State memory \\
Full Recurrent & \checkmark & \checkmark & \checkmark & \checkmark & \checkmark & All connections \\
\bottomrule
\end{tabular}
\end{table}

Cross-problem evaluation (depth~5, pop~300, 10 replications) shows connection types significantly affect convergence (Table~\ref{tab:cross_problem}). On XOR, Backward and Lateral achieve 100\% ($p < 0.001$ vs Feedforward's 20\%). On Visual Discrimination (2D spatial classification)~\cite{stanley2009hypercube}, Backward achieves 100\% convergence while Lateral and Self hurt performance.

\begin{table}[ht]
\centering
\caption{Cross-problem convergence rate (\%) by connection type. Bold = best or tied. $p$: Fisher's exact, best vs feedforward.}
\label{tab:cross_problem}
\small
\begin{tabular}{lccccccc}
\toprule
Problem & FF & \hh{} & Back & Lat & Self & Full & $p$ \\
\midrule
XOR & 20 & 50 & \textbf{100} & \textbf{100} & 90 & 60 & $<$0.001 \\
Visual & 90 & 70 & \textbf{100} & 40 & 40 & 90 & $<$0.001 \\
Retina & \textbf{100} & 80 & 90 & 80 & \textbf{100} & \textbf{100} & n.s. \\
\bottomrule
\end{tabular}
\end{table}

\section{Experimental Evaluation}
\label{sec:experiments}

\subsection{Setup}

CPU: Apple M4 Max (JAX CPU backend). GPU: NVIDIA RTX 2080 Ti 11GB (JAX CUDA backend). Baseline: PUREPLES~\cite{westh2017pureples}, the reference Python implementation of ES-HyperNEAT (ES-HN). Problem: XOR (4 input patterns), standard neuroevolution benchmark~\cite{stanley2002evolving}. Both algorithms: 30 generations, 8 populations (50--1000), depths 1--7, \hh{} connection type, 3 replications (504 runs total). \emrhn{} implemented in JAX/Python on Tensor\-NEAT~\cite{wang2024tensorized}; each experiment ran in process isolation to ensure timing includes JIT compilation overhead.

\subsection{Scaling and Speedup}

Table~\ref{tab:speedup} presents per-generation timing. Figure~\ref{fig:total_runtime} shows total runtime scaling including depths 8--13 via memory streaming, with cumulative runtime IQR median approaching ${\sim}100\times$ over 30 generations at depth~7.

\begin{table}[t]
\centering
\caption{Average time per generation (XOR, pop=1000, \hh{} type). Speedup excludes one-time JIT compilation.}
\label{tab:speedup}
\small
\resizebox{\columnwidth}{!}{%
\begin{tabular}{lcccc}
\toprule
Depth & ES-HN & EMR & JIT & Speedup CPU/GPU \\
\midrule
1--3 & 0.3--1.0s & 1--3s / 0.9--2.3s & 4--8s / 7--17s & 0.2--0.3$\times$ \\
4 & 4.27s & 5s / 2.9s & 12s / 18s & 0.85$\times$ / 1.5$\times$ \\
5 & 46.3s & 10s / 3.9s & 15s / 19s & 4.6$\times$ / 12$\times$ \\
6 & 200.4s & 28s / 8.9s & 30s / 24s & 7.2$\times$ / 23$\times$ \\
7 & 2,020s & 165s / 60s$^\dagger$ & 60s / 39s & 12.2$\times$ / 34$\times$$^\dagger$ \\
\bottomrule
\multicolumn{5}{l}{\footnotesize $^\dagger$GPU depth 7 uses memory streaming (50--70s/gen on 11GB VRAM).}
\end{tabular}}
\end{table}

\begin{figure*}[!t]
\centering
\begin{tikzpicture}
\begin{axis}[
    width=\textwidth,
    height=5.75cm,
    ylabel={Total Runtime},
    xlabel={Depth},
    ymode=log,
    ytick={1, 10, 60, 600, 3600, 36000, 86400, 604800, 1209600, 2592000, 7776000},
    yticklabels={1s, 10s, 1min, 10min, 1h, 10h, 1d, 1w, 2w, 1mo, 3mo},
    ymin=0.5,
    ymax=12000000,
    xmin=0.5,
    xmax=13.5,
    xtick={1,2,3,4,5,6,7,8,9,10,11,12,13},
    xticklabels={1,2,3,4,5,6,7,8,9,10,11,12,13},
    ymajorgrids=true,
    grid style={dashed, gray!30},
    tick label style={font=\small},
    label style={font=\small},
    legend style={at={(0.02,0.98)}, anchor=north west, font=\scriptsize},
]
\addplot[name path=eshnmin, color=oivermillion, thick, dotted] coordinates {
    (1, 1.9) (2, 2.5) (3, 8.6) (4, 49.6) (5, 191.5) (6, 2716.6) (7, 13456.5)
};
\addplot[name path=eshnmax, color=oivermillion, thick, dotted] coordinates {
    (1, 6.0) (2, 7.5) (3, 20.2) (4, 81.1) (5, 492.8) (6, 6100.4) (7, 30703.5)
};
\addplot[fill=oivermillion, fill opacity=0.2] fill between[of=eshnmin and eshnmax];
\addplot[name path=cpumin, color=oigreen, thick, dashed] coordinates {
    (1, 17.9) (2, 22.5) (3, 26.1) (4, 34.1) (5, 64.0) (6, 207.5) (7, 729.8)
};
\addplot[name path=cpumax, color=oigreen, thick, dashed] coordinates {
    (1, 24.5) (2, 47.2) (3, 92.6) (4, 112.3) (5, 225.1) (6, 802.2) (7, 3277.9)
};
\addplot[fill=oigreen, fill opacity=0.3] fill between[of=cpumin and cpumax];
\addplot[name path=gpumin, color=oiblue, thick, dashed] coordinates {
    (1, 17.7) (2, 26.8) (3, 31.8) (4, 34.1) (5, 41.4) (6, 74.3) (7, 155.1)
    (8, 830) (9, 2599) (10, 11230) (11, 37764) (12, 147552) (13, 5467749)
};
\addplot[name path=gpumax, color=oiblue, thick, dashed] coordinates {
    (1, 50.4) (2, 82.7) (3, 77.6) (4, 116.4) (5, 110.3) (6, 182.2) (7, 590.9)
    (8, 2494) (9, 9166) (10, 31359) (11, 115239) (12, 605310) (13, 7388913)
};
\addplot[fill=oiblue, fill opacity=0.3] fill between[of=gpumin and gpumax];
\addplot[name path=cpujitmin, color=oigreen, thick, dotted] coordinates {
    (1, 3.7) (2, 4.3) (3, 5.4) (4, 5.7) (5, 6.8) (6, 8.9) (7, 19.6)
};
\addplot[name path=cpujitmax, color=oigreen, thick, dotted] coordinates {
    (1, 5.3) (2, 6.1) (3, 13.2) (4, 12.1) (5, 15.4) (6, 30.9) (7, 78.0)
};
\addplot[fill=oigreen, fill opacity=0.15, forget plot] fill between[of=cpujitmin and cpujitmax];
\addplot[name path=gpujitmin, color=oiblue, thick, dotted] coordinates {
    (1, 5.4) (2, 6.2) (3, 8.2) (4, 7.8) (5, 7.9) (6, 9.0) (7, 13.2)
    (8, 36.0) (9, 89.0) (10, 372) (11, 1171) (12, 4355) (13, 173291)
};
\addplot[name path=gpujitmax, color=oiblue, thick, dotted] coordinates {
    (1, 6.4) (2, 7.7) (3, 12.1) (4, 12.2) (5, 11.5) (6, 15.5) (7, 25.3)
    (8, 82.7) (9, 280) (10, 970) (11, 3434) (12, 20563) (13, 241994)
};
\addplot[fill=oiblue, fill opacity=0.15, forget plot] fill between[of=gpujitmin and gpujitmax];
\addplot[color=oivermillion, thick, mark=diamond*, mark size=2, dashed] coordinates {
    (1, 57.0) (2, 60.0) (3, 72.0) (4, 126.0) (5, 657.0) (6, 7437.0)
};
\addplot[color=oiblue, thick, mark=pentagon*, mark size=2, dashed] coordinates {
    (1, 138.7) (2, 144.8) (3, 157.5) (4, 163.9) (5, 579.5) (6, 1124.4)
};
\draw[thick, gray] (axis cs:7, 0.5) -- (axis cs:7, 50000);
\draw[thick, gray, decorate, decoration={brace, amplitude=4pt}]
    (axis cs:1, 0.8) -- (axis cs:7, 0.8);
\node[font=\scriptsize, gray] at (axis cs:4, 1.8) {on-device computation};
\draw[thick, gray, decorate, decoration={brace, amplitude=4pt}]
    (axis cs:7, 0.8) -- (axis cs:13, 0.8);
\node[font=\scriptsize, gray] at (axis cs:10, 1.8) {memory streaming};
\draw[->, >=stealth, thick, black] (axis cs:1.5, 4) -- (axis cs:6.5, 18000);
\node[font=\tiny, black, rotate=24, above] at (axis cs:4.2, 400)
    {$\sim$5000$\times$ ($\approx 4^D$)};
\draw[->, >=stealth, thick, black] (axis cs:1.5, 10) -- (axis cs:6.5, 24);
\node[font=\tiny, black, rotate=2] at (axis cs:4.2, 25)
    {$\sim$7$\times$ (near-linear)};
\draw[->, >=stealth, thick, black] (axis cs:7.5, 50) -- (axis cs:12.5, 100000);
\node[font=\tiny, black, rotate=22, above] at (axis cs:10, 1600)
    {$\sim$4000$\times$ ($\approx 4^D$)};
\draw[->, >=stealth, thick, black]
    (axis cs:11.3, 6000000) -- (axis cs:12.9, 6000000);
\node[font=\scriptsize, black, anchor=east] at (axis cs:11.25, 6000000)
    {disk I/O};
\legend{ES-HN, , , EMR (CPU), , , EMR (GPU), , , CPU JIT, , GPU JIT, , CCM ES-HN, CCM EMR}
\end{axis}
\end{tikzpicture}\vspace{-1em}
\caption{Runtime over 30 generations (log y-axis). IQR bands: d1--7 use 8 pops $\times$ 3 seeds; d8--13 use GPU$\leftrightarrow$RAM streaming (pop 300, 1000, 3 seeds). The d13 widening reflects disk I/O dominating JIT and per-gen cost. CCM lines: financial data~\cite{crsp_compustat_ccm} (94K samples, pop 100). Trend arrows: ${\sim}5000\times$ ES-HN ($\approx 4^D$), ${\sim}7\times$ EMR on-device, ${\sim}4000\times$ EMR streaming ($\approx 4^D$).}
\label{fig:total_runtime}
\Description{Log-scale runtime comparison from 1 second to 3 months across depths 1--13. ES-HyperNEAT grows exponentially and becomes infeasible past depth~7. EMR GPU scales near-linearly on-device for depths 1--7, then follows 4-to-the-D scaling under GPU-RAM streaming for depths 7--12. At depth 13, disk I/O becomes the dominant cost, shown by an annotation arrow. Bands show IQR across 8 populations and 3 replications.}
\vspace{-1em}
\end{figure*}
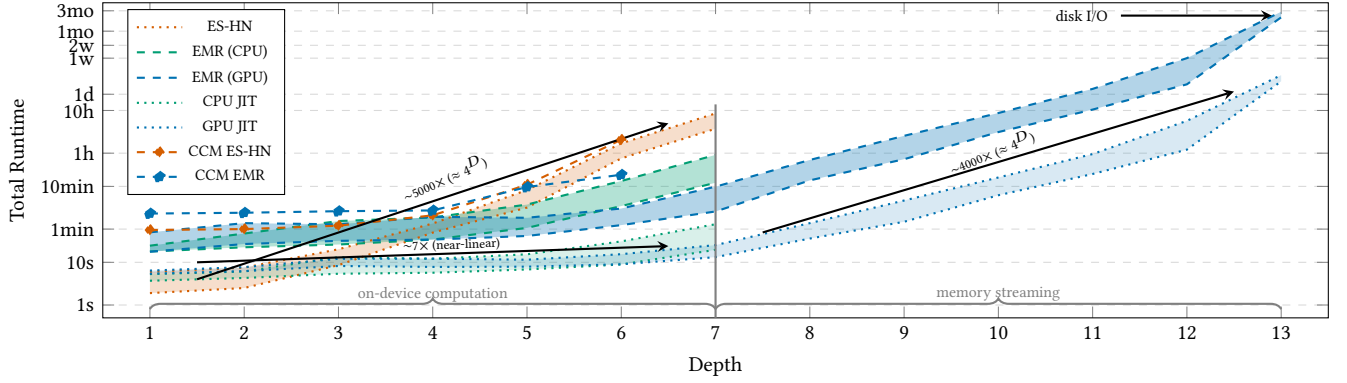

Statistical validation (Table~\ref{tab:statistical}) confirms that at depths $\geq$5, \emrhn{} GPU achieves significant speedups ($p < .001$, Cliff's $\delta \geq 0.99$, large effect) after Holm-Bonferroni correction. Speedup increases with depth (Spearman $\rho = 0.96$, $p < .001$). Runtime variance also diverges: at depth~7, estimated~$\sigma$ is 3.6h~(ES-HN), 31min~(EMR~CPU), 1.9min~(EMR~GPU), a 114$\times$ tighter bound.

\begin{table}[t]
\centering
\caption{Statistical comparison vs ES-HN: Mann-Whitney U (one-sided), Cliff's delta. Significant at $\alpha = 0.05$ for GPU depths $\geq 5$ after Holm-Bonferroni correction.}
\label{tab:statistical}
\small
\begin{tabular}{ccccccc}
\toprule
& \multicolumn{3}{c}{EMR GPU} & \multicolumn{3}{c}{EMR CPU} \\
\cmidrule(lr){2-4} \cmidrule(lr){5-7}
Depth & $p$ & $\delta$ & Speedup & $p$ & $\delta$ & Speedup \\
\midrule
1--3 & $>$0.99 & $-$1.00$^{L}$ & 0.1--0.3$\times$ & $>$0.99 & $-$1.00$^{L}$ & 0.1--0.3$\times$ \\
4 & $<$.001 & 0.59$^{L}$ & \textbf{1.6$\times$} & .915 & -0.23$^{S}$ & 0.8$\times$ \\
5 & $<$.001 & 0.99$^{L}$ & \textbf{6.8$\times$} & $<$.001 & 0.70$^{L}$ & \textbf{3.5$\times$} \\
6 & $<$.001 & 1.00$^{L}$ & \textbf{42$\times$} & $<$.001 & 0.98$^{L}$ & \textbf{11$\times$} \\
7 & $<$.001 & 1.00$^{L}$ & \textbf{75$\times$} & $<$.001 & 0.91$^{L}$ & \textbf{15$\times$} \\
\bottomrule
\multicolumn{7}{l}{\footnotesize Effect size: S=small, L=large.} \\
\end{tabular}
\end{table}

\subsection{Solution Quality}

Table~\ref{tab:pop_depth} shows \emrhn{} achieves equal or higher solve rates. At depth~6, ES-HyperNEAT solves XOR in only 33\% of runs (pop=300) while \emrhn{} achieves 100\%, due to bottom-up evaluation recovering positions that top-down traversal misses. Both achieve identical maximum fitness (0.99+) when successful.
\begin{table}[ht]
\centering
\caption{Per-generation time and solve rate (XOR, \hh{} type with cache). JIT = one-time compilation overhead.}
\label{tab:pop_depth}
\small
\begin{tabular}{llrrrr}
\toprule
Depth & Algorithm & Pop 300 & + JIT & Pop 500 & + JIT \\
\midrule
\multirow{3}{*}{4} & ES-HN & 1.5s / 67\% & - & 2.5s / 33\% & - \\
 & EMR (CPU) & 2s / 100\% & 11s & 3s / 100\% & 10s \\
 & EMR (GPU) & 0.9s / 100\% & 10s & 1.3s / 100\% & 12s \\
\midrule
\multirow{3}{*}{5} & ES-HN & 6.8s / 0\% & - & 22.2s / 100\% & - \\
 & EMR (CPU) & 6s / 100\% & 15s & 7s / 100\% & 14s \\
 & EMR (GPU) & 1.2s / 100\% & 10s & 1.9s / 100\% & 13s \\
\midrule
\multirow{3}{*}{6} & ES-HN & 147.8s / 33\% & - & 220.5s / 0\% & - \\
 & EMR (CPU) & 12s / 100\% & 16s & 40s / 100\% & 31s \\
 & EMR (GPU) & 2.6s / 100\% & 12s & 4.3s / 100\% & 15s \\
\midrule
\multirow{3}{*}{7} & ES-HN & 652s / 0\% & - & 1,649s / 0\% & - \\
 & EMR (CPU) & 48s / 100\% & 46s & 94s / 100\% & 72s \\
 & EMR (GPU) & 7.9s / 100\% & 17s & 13s / 100\% & 23s \\
\bottomrule
\end{tabular}
\end{table}

\subsection{Real-World Benchmark}

Table~\ref{tab:ccm_comparison} compares performance on CRSP/Compustat Merged financial data (94,464 observations)~\cite{crsp_compustat_ccm}. At depth~6, \emrhn{} is 5.5$\times$ faster despite streaming deeper layers from CPU~RAM when GPU memory is insufficient. ES-Hyper\-NEAT exhibits exponential scaling (depth~5$\to$6: 11.3$\times$ slower); \emrhn{} scales linearly on-device.

\begin{table}[t]
\centering
\caption{Per-generation time on CCM financial dataset (94,464 samples, pop=100). EMR uses \hh{} type (GPU).}
\label{tab:ccm_comparison}
\small
\begin{tabular}{rrrrr}
\toprule
Depth & ES-HN & EMR & EMR JIT & Speedup \\
\midrule
1--3 & 1.9--2.4s & 5.5--6.2s & 11--13s & 0.35--0.39$\times$ \\
4 & 4.2s & 6.6s & 13s & 0.64$\times$ \\
5 & 21.9s & 23s & 46s & 0.95$\times$ \\
6 & 247.9s & 45s$^\dagger$ & 89s & \textbf{5.5$\times$} \\
\bottomrule
\multicolumn{5}{l}{\footnotesize $^\dagger$Streams deeper layers from CPU RAM (GPU memory insufficient).}
\end{tabular}
\end{table}

\section{Discussion}
\label{sec:discussion}

The threshold adaptation (0.03 vs 0.5) is a necessary consequence of the reformulation: ES-HyperNEAT's original thresholds produce no solutions under eager evaluation because variance is computed over all positions rather than only inside already-accepted regions. The adapted thresholds produce more hidden node candidates, contributing to the higher solve rates at depths 5--7 (Table~\ref{tab:pop_depth}), at the cost of redundant CPPN queries amortized by parallelism. \emrhn{} is advantageous when JIT overhead (4--60s) is amortized across generations. JAX caches compiled functions, so experiments at matching depth and population size skip JIT, and the precomputed grid (${\sim}$683~KB at depth~7) is reused across problems. At low depth with few generations, ES-HyperNEAT remains preferable.

Caching accelerates computation further: cached \hh{} connections reduce subsequent discovery from ${\sim}$27s to ${\sim}$3ms. Memory scales as $4^D$: depth~5 $<$620~MB, depth~7 ${\sim}$10~GB (pop~500). At depth~13 (358M positions, pop~300), weight arrays stored on disk dominate both JIT (${\sim}$5.3~h) and per-gen cost (${\sim}$5.8~h), visible as the d12$\to$d13 widening in Figure~\ref{fig:total_runtime}.
\section{Conclusion}
\label{sec:conclusion}

To our knowledge, \emrhn{} is the first massively parallel adaptive substrate discovery algorithm. The lazy-to-eager reformulation yields significant speedups that grow with depth (Tables~\ref{tab:speedup}--\ref{tab:ccm_comparison}), validated on both XOR and financial data (94K samples). Beyond speedup, bottom-up evaluation discovers a superset of substrate positions (Theorem~\ref{thm:correctness}), which yields higher solve rates in practice, and the connection type taxonomy makes recurrent substrates computationally tractable. This makes previously unreachable substrate configurations accessible, opening per-node substrate complexification at new depth and population scales, though further efficiency gains via selective evaluation remain constrained by the need to preserve parallelization.

\section*{Code Availability}

\url{https://github.com/RomainClaret/emr-hyperneat}

\begin{acks}
Generative AI tools assisted with visualizations, pseudocode conversion, and draft revision.
\end{acks}

\bibliographystyle{ACM-Reference-Format}
\bibliography{references}

\received{19 January 2026}
\received[revised]{10 March 2026}
\received[accepted]{27 January 2026}

\end{document}